\documentclass[journal]{IEEEtran}
\newfont{\my}{cmr10 scaled \magstep 0}

\usepackage{cite}
\usepackage{multirow}
\ifCLASSINFOpdf
  \usepackage[pdftex]{graphicx}
\else
  \usepackage[dvips]{graphicx}
\fi

\usepackage{url}

\usepackage{amsmath}
\usepackage{amsthm}
\usepackage{amsfonts}
\usepackage{amssymb}
\usepackage{subcaption}

\begin{document}

\newtheorem{theorem}{Theorem}
\newtheorem{lemma}[theorem]{Lemma}
\newtheorem{definition}[theorem]{Definition}

\newcommand{\R }{\mathbb{R}}
\newcommand{\I }{\mathcal{I}}
\newcommand{\A }{\mathcal{A}}
\newcommand{\B }{\mathcal{B}}
\newcommand{\F }{\mathcal{F}}
\newcommand{\Ss }{\mathcal{S}}
\newcommand{\x }{\mathbf{x}}
\newcommand{\kk }{\mathbf{k}}
\newcommand{\nn }{\mathbf{n}}
\newcommand{\pp }{\mathbf{p}}
\newcommand{\uu }{\mathbf{u}}
\newcommand{\ud }{\,\mathrm{d}}

\title{Blur Invariants for Image Recognition}

\author{Jan Flusser,
Mat\v{e}j L\'ebl, Matteo Pedone,  Filip \v{S}roubek, and Jitka Kostkov\'a
\thanks{Jan~Flusser, Mat\v{e}j L\'ebl, Filip \v{S}roubek, and Jitka Kostkov\'{a} are with the Czech Academy of Sciences, Institute of Information Theory and Automation, Pod vod\'{a}renskou v\v{e}\v{z}\'{\i} 4, 182\,08 Praha 8, Czech Republic,

e-mails: \{flusser, lebl, sroubekf, kostkova\}@utia.cas.cz}
\thanks{Matteo Pedone is with the Center for Machine Vision Research, Department of Computer Science and Engineering, University of Oulu, Oulu FI-90014, Finland 

e-mail: matteo.pedone@oulu.fi}

}


\maketitle

\begin{abstract}
    Blur is an image degradation that is difficult to remove. Invariants with respect to blur offer an alternative way of a~description and recognition of blurred images without any deblurring. In this paper, we present an original unified theory of blur invariants. Unlike all previous attempts, the new theory does not require any prior knowledge of the blur type. The invariants are constructed in the Fourier domain by means of orthogonal projection operators and moment expansion is used for efficient and stable computation.
    It is shown that all blur invariants published earlier are just particular cases of this approach. 
    Experimental comparison to concurrent approaches shows the advantages of the proposed theory.
\end{abstract}

\begin{IEEEkeywords}
Blurred image, object recognition, blur invariants, projection operators, moments.
\end{IEEEkeywords}
\IEEEpeerreviewmaketitle

\section{Introduction}

In image processing and analysis, we often have to deal with images that are degraded versions of the original scene. One of the most common degradations is \emph{blur}, which usually appears as smoothing or suppression of high-frequency details of the image. Capturing an ideal scene $f$ by an imaging device with the point-spread function (PSF) $h$, the observed image $g$ can be modeled as a~convolution of both
\begin{equation} \label{model1}
   g(\x) = (f*h) (\x)\,.
\end{equation}
This linear image formation model, even if it is very simple, is a~reasonably accurate approximation of many imaging devices and acquisition scenarios.

The blur may come from various physical sources. 
Based on our prior knowledge about the PSF, we distinguish a~\emph{blind} case when no information about the PSF is available, a~\emph{semi-blind} case when some (incomplete) information about the PSF is available (for instance its parametric form), and a~\emph{non-blind} case, when the PSF is known completely.

In classical image processing monographs~\cite{pratt, gonzales}, the first methods of solving Eq.~\eqref{model1} for $f$ were proposed for the non-blind case. The semi-blind and blind cases are much more difficult. Despite their extensive study (see~\cite{kund1, hnedakniha, motiondeblurring} for a~survey), they have not been fully resolved yet. Although some of the current image deconvolution methods yield good results, they rely on prior knowledge incorporated into regularization terms or other constraints. If such prior knowledge is not available, the methods may converge to solutions that are far from the ground truth. If noise is present, the inverse problem becomes even more ill posed and its solution numerically less stable.

\begin{figure}[htb]
    \centering
    \includegraphics[width=\linewidth]{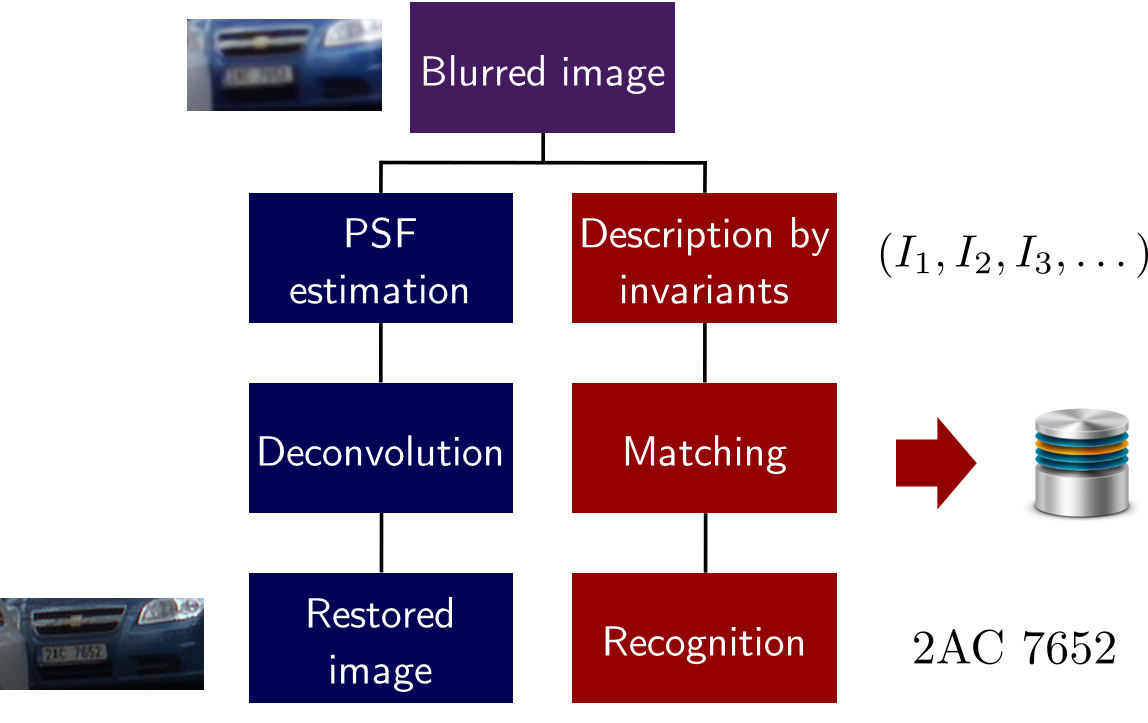}
    \caption{The flowchart of image restoration (left) and of the recognition by blur invariants (right).}
    \label{Fig:flowchartblur}
\end{figure}

In the 1990s, some researchers not only realized all the above-mentioned difficulties connected with the solving of Eq.~\eqref{model1} but also found out that in many applications a~complete restoration of~$f$ is not necessary and can be avoided, provided that an appropriate image representation is used. A~typical example is a~recognition of objects and patterns in blurred images, where a~blur-robust object description forms a~sufficient input for the classifier (see Fig.~\ref{Fig:flowchartblur} for the illustration of the difference between the recognition and restoration approaches). This led to the introduction of the idea of \emph{blur invariants}, which are powerful in many semi-blind cases. Roughly speaking, blur invariant~$I$ is a~functional fulfilling the constraint $I(f) = I(f*h)$ for any~$h$ from a~certain set $\Ss$ of admissible PSFs. Many systems of blur invariants have been proposed so far (see~\cite{2D3D}, Chapter 6 and further references thereof). They differ from one another by the assumptions on the PSF, by the mathematical tools used for invariant construction, by the domain in which the invariants are defined, and by the application area for which the invariants were designed.

The main drawback of all current blur invariants is that they lack a~unified mathematical framework. For each class of PSFs, the invariants had to be derived ``from scratch'', which means that one had to prove the invariance property for each PSF type separately.
Although one can re-use similar calculation techniques for various families of PSFs,
both explicit derivation and formal proof of invariance always had to be customized for any particular family of the PSFs.

We discover a~unified theoretical background of blur invariants which is presented in this paper for the first time. We show that all previously published blur invariants are particular cases of a~general theory, which provides this topic with a~roof. Two key theorems, referred here as Theorem~\ref{FTblur_invariants_general} and Theorem~\ref{completeness_theorem}, are formulated and proved here \emph{regardless of the particular PSF type}. This is a~significant theoretical contribution of this paper, which has an immediate practical consequence. If we want to derive blur invariants w.r.t. a~new class of PSFs, Theorems~\ref{FTblur_invariants_general} and~\ref{completeness_theorem} offer the solution directly, provided that the PSF in question complies with the assumption of the Theorems. Verifying that is, however, much easier than the construction of the invariants from the beginning.


\subsection{State of the art of blur invariants}
\label{State_of_the_art_of_blur_invariants}

Unlike geometric invariants, which can be traced over two centuries back to Hilbert~\cite{hilbert}, blur invariants are a~relatively new topic. The problem formulation and the basic idea appeared originally in the '90s in the series of papers by Flusser et al.~\cite{FluSuk:95a, FluSuk:96a, FluSuk:980022}. The invariants presented in these pioneer papers were found heuristically without any theoretical background. The authors observed that certain moments of a~symmetric PSF vanish. They derived the relation between the moments of the blurred image and the original and thanks to the vanishing moments of the PSF they eliminated the non-zero PSF moments by a~recursive subtraction and multiplication. They did it for axially symmetric~\cite{FluSuk:95a, FluSuk:96a} and centrosymmetric~\cite{FluSuk:980022} PSFs. These invariants, despite their heuristic derivation and the restriction to centrosymmetric PSFs, have been adopted by many researchers in further theoretical studies~\cite{hui:Legendre, wee_Legendre, xiubin:Legendre, xiubin,xiubin:pseudo_Zernike, cheb-cabmi, xin-SIFT, KauFlu:11, zhang, ZhangPR02, FluBolZit, FluBol:pami03, candocia, Ojansivu1,ojansivu:spl, makaremi:wvl, makaremi:tip, gali:radon, Ojansivu_Heikkila_LPQ1} 
and in many application-oriented papers~\cite{bentoutours, zhaoxia, huetal, bentoutou, bentoutou3D, bentoutou-CVIU, bob, Ahonen_Rahtu_Ojansivu_Heikkila_LPQ4, zhangfuze, yap}.

By a~similar heuristic approach, various invariants to circularly symmetric blur~\cite{FluZit:icpr04, zhu:ZernikePAA, beijing:Zernike, hanjie:Zernike, xiubin:pseudo_Zernike, cheb-cabmi, FMM}, linear motion blur~\cite{FluSuk:96, stern, weed, woodslice, water:TIP, guan-biologically, wang-sinusoidal}, and Gaussian blur~\cite{tianxu:gauss, xiao-gauss} have been proposed.

Significant progress of the theory of blur invariants was made by Flusser et al. in~\cite{pami:2015}, where the invariants to arbitrary $N$-fold rotation symmetric blur were proposed. In that paper, a~derivation based on a~mathematical theory rather than on heuristics was presented for the first time. The invariants were constructed by means of projection of the blurred image onto the subspace of the PSFs. A~similar result achieved by another technique was later published by Pedone et al.~\cite{matteo:tip}.

The main limitation of all the above mentioned methods is their restriction to a~given single class of blurs. In other words, the authors first defined the blur type they were considering and then they derived the invariants based on the specific properties of the blur. In this paper, we approach the problem the other way round. Regardless of the particular blur, we find a~general formula for blur invariants. Then, for any type of admissible PSF, this general formula immediately provides specific invariants. This is the major original contribution of this paper that differentiates the proposed theory from the previous ones.

\section{Mathematical preliminaries}
\label{Preliminaries}

\begin{definition} \label{image_function}
    By an \emph{image function} (or \emph{image}) we understand any real function $f \in L_1 \left( \R^d \right) \cap L_2 \left( \R^d \right)$ with a~compact support.\footnote{Symbol  $L_p \left( \R^d \right)$ denotes a space of all functions of $d$ real variables such that $\int |f|^p < \infty$.} The set of all image functions is denoted as $\I$.
\end{definition}

For the convenience, we assume the Dirac $\delta$-function to be an element of $\I$.\footnote{From mathematical point of view, this is formally incorrect since $\delta \notin L_1 \cap L_2$. We could correctly include $\delta$ by means of theory of distributions but this would be superfluous for the purpose of this paper.}
    

\begin{definition} \label{moments}
    Let $\B = \{\pi_{\kk}(\x)\}$ be a~set of $d$-variable polynomials. Then the integral
    \begin{equation}
        M_{\pp}^{(f)} = \int \pi_{\pp}(\x) f(\x) \ud \x
    \end{equation} 
    is called \emph{moment} of function $f$ with respect to set $\B$. The non-negative integer $\mathbf{|p|}$, where $\pp$ is a~$d$-dimensional multi-index, is the \emph{order} of the moment.
\end{definition}

Moments are widely used descriptors of compactly-supported functions. Depending on $\B$, we recognize various types of moments. If $\pi_{\kk}(\x) = \x^{\kk}$ we speak about \emph{geometric} moments. If $d=2$ and $\pi_{pq}(x,y) = (x+iy)^p(x-iy)^q$, we obtain \emph{complex} moments. If the polynomials $\pi_{\kk}(\x)$ are orthogonal (or orthogonal with a~weight), we get \emph{orthogonal} (OG) moments. Legendre, Zernike, Chebyshev, and Fourier-Mellin moments are the most common examples. For the theory of moments and their application in image analysis we refer to~\cite{2D3D}.

\begin{definition} \label{projector}
    A linear operator $P: \I \to \I$ is called a~\emph{projection operator} (or \emph{projector} for short) if it is idempotent, i.e. $P^2 = P$.
\end{definition}

The image space $\I$ and any projector $P$ satisfy the following Lemma.
\begin{lemma} \label{lemma1}
The following statements hold:
    \begin{enumerate}
        \item 
            $\I$ is dense both in $L_1$ and $L_2$. 
            
            \emph{This proposition shows that the image space is sufficiently ``large''.}
        \item 
            For any $f,g \in \I$ their convolution exists and ${f*g \in \I}$.
            
            \emph{This \emph{convolution closure} property follows from Young's inequality.} 
        \item 
            For any $f \in \I$, its Fourier transform $\F(f)$ exists.
        \item 
            For any $f \in \I$, its moments w.r.t. arbitrary $\B$ exist and are finite.
        \item  
            Let $\Ss = P(\I)$. The set $\Ss$ is also a~vector space and $\I$ can be expressed as a~direct sum ${\I = \Ss \oplus \A }$, where $\A$ is called the \emph{complement} of $\Ss$ and is also a~vector space.
            Any $f \in \I$ can be unambiguously written as a~sum $f = Pf + f_A$, where $Pf$ is a~projection of $f$ onto $\Ss$ and $f_A \in \A$ is simply defined as $f_A = f - Pf$.
        \item 
            For any $f \in \I$, $Pf_A = 0$ and, consequently, $\Ss \cap \A = \{0\}$. If $f \in \Ss$ then $f = Pf$ and vice versa.
        
    \end{enumerate}
\end{lemma}

\begin{definition} \label{OG projector}
    Projector $P$ is called \emph{orthogonal} (OG), if the respective subspaces $\Ss$ and $\A$ are orthogonal.
\end{definition}


\section{Blur Invariants}
\label{Sec. Blur Invariants}

In this Section, we show how blur invariants can be constructed by means of suitable projectors. Let $\Ss$ be, from now on, the set of blurring functions (PSFs), with respect to which we want to design the invariants. Any meaningful $\Ss$ must contain at least one non-zero function and must be closed under convolution. In other words, for any $h_1, h_2 \in \Ss$ must be $h_1 * h_2 \in \Ss$. This is the basic assumption without which the question of invariance does not make sense. If $\Ss$ was not closed under convolution, then any potential invariant would be in fact invariant w.r.t. convolution with functions from the ``convolution closure'' of $\Ss$, which is the smallest superset of $\Ss$ closed to convolution.

Under the closure assumption, $(\Ss,*)$ forms a~commutative \emph{semi-group} (it is not a~group because the existence of inverse elements is not guaranteed). Hence, the convolution may be understood as a~\emph{semi-group action} of $\Ss$ on $\I$.  The convolution defines the following equivalence relation on $\I$: $f \sim g$ if and only if there exist $h_1, h_2 \in \Ss$ such that $h_1*f = h_2*g$. Thanks to the closure property of $\Ss$ and to the commutativity of convolution, this relation is transitive, while symmetry and reflexivity are obvious. This relation factorizes $\I$ into classes of \emph{blur-equivalent images}. In particular, all elements of $\Ss$ are blur-equivalent. 
The image space partitioning and the action of the projector are visualized in Figure~\ref{fig:projection}.

\begin{figure}[!htbp]
   \centering
    \includegraphics[width=\linewidth]{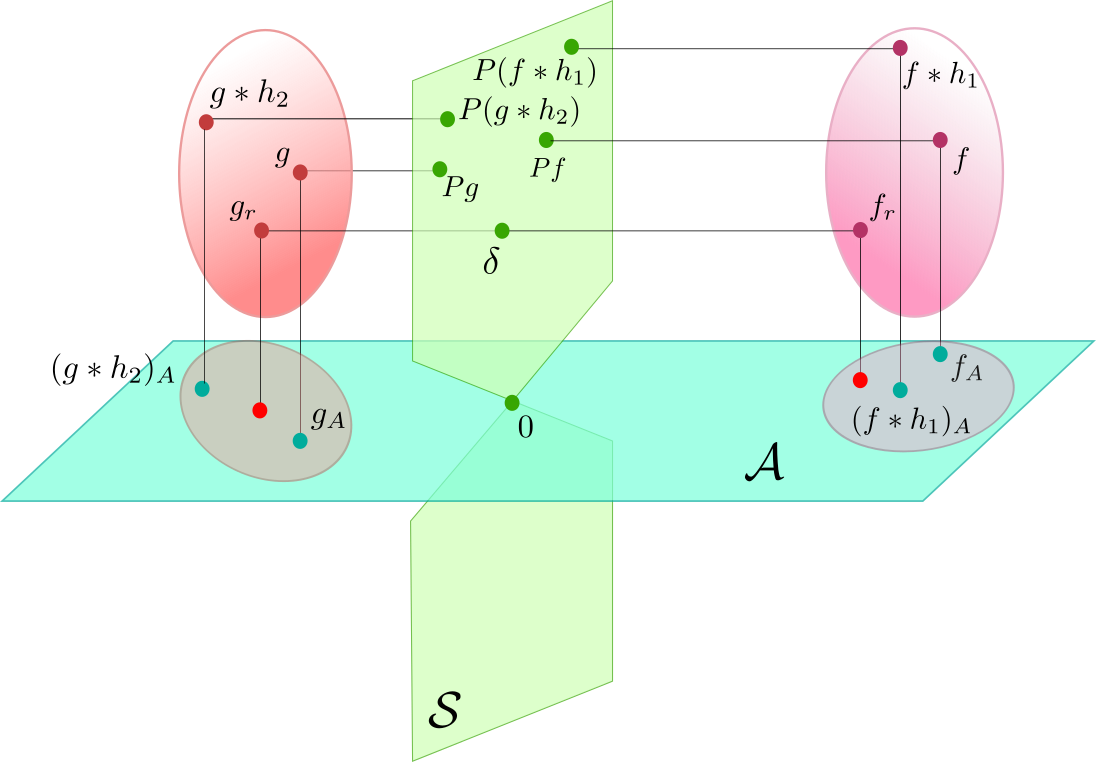}
    \caption{Partitioning of the image space. Projection operator $P$ decomposes image $f$ into its projection $Pf$ onto $\Ss$ and its complement~$f_A$, which is a~projection onto $\A$. The ellipsoids depict blur-equivalent classes. Blur-invariant information is 
    contained in the primordial images $f_r$ and $g_r$ (see the text).}
    \label{fig:projection}
\end{figure}




Now we are ready to formulate the following \emph{General theorem of blur invariants} (GTBI), which performs the main contribution of the paper and a~significant difference from all previous work on this field.

\begin{theorem}[GTBI] \label{FTblur_invariants_general}
    Let $\Ss$ be a~linear subspace of $\I$, which is closed under convolution and correlation. Let $P$ be an orthogonal projector of $\I$ onto $\Ss$. Then
    \begin{equation}
        I(f)(\uu) \equiv \frac{\F(f)(\uu)}{\F(Pf)(\uu)}
    \end{equation}
    is an invariant w.r.t. a~convolution with arbitrary $h \in \Ss$ at all frequencies $\uu$ where $I(f)$ is well defined.
\end{theorem}

\begin{proof}
    Let us assume that $P$ is ``distributive'' over a~convolution with functions from $\Ss$, which means
    $P(f*h) = Pf * h$ for arbitrary $f$ and any $h \in \Ss$. Then the proof is trivial,
    we just employ the basic properties of Fourier transform:
    \begin{align}
        I(f*h) \equiv & \ \frac{\F(f*h)}{\F(P(f*h))} = \frac{\F(f) \cdot \F(h)}{\F(Pf*h)} = \nonumber \\
        = & \frac{\F(f) \cdot \F(h)}{ \F(h) \cdot \F(Pf)} = I(f)\,. 
    \end{align} 
    
    The ``distributive property" of $P$ is equivalent to the constraint that the complement $\A$ is closed w.r.t. convolution with functions from $\Ss$. This follows from
    \begin{align} \label{proj-h1}
        P(f*h) = & P((Pf + f_A)*h) = P(Pf*h + f_A*h)  \nonumber \\
         = & Pf*h + P(f_A*h)\,.
    \end{align}
    
    Now let us show that this constraint is implied by  the orthogonality of $P$ regardless of its particular form.
 
    Since $\Ss \perp \A$ and Fourier transform on $L_1 \cap L_2$ preserves the scalar product (this property is known as Plancherel Theorem), then $\F(\Ss) \perp \F(\A)$. Let us consider arbitrary functions $a \in \A$ and $h_1,h_2 \in \Ss$. Using the Plancherel Theorem, the convolution theorem,
    and the correlation theorem (the correlation $\circledast$ of two functions is just a~convolution with a~flipped function), we have
    \begin{align}
        \langle a*h_1,h_2 \rangle = &\langle \F(a*h_1), \F(h_2) \rangle = \int A \cdot H_1 \cdot H_2^* = \nonumber \\
         = & \langle A, H_1^* \cdot H_2 \rangle = \langle a, h_2 \circledast h_1 \rangle = 0 \,.
    \end{align}
    The last equality follows from the closure of $\Ss$ w.r.t. correlation. Hence, $\A$ has been proven to be closed w.r.t. convolution with functions from $\Ss$, which completes the entire proof.
\end{proof}

The invariant $I(f)$ is not defined if $Pf = 0$, which means this Theorem cannot be applied if $f \in \A$. In all other cases, $I(f)$ is well defined almost everywhere. Since $\Ss$ contains compactly-supported functions only, $\F(Pf)(\uu)$ cannot vanish on any open set and therefore the set of frequencies, where $I(f)$ is not defined, has a~zero measure.\footnote{This may not be true if $\I$ and $\Ss$ contained functions of unlimited support. Then $\F(Pf)(\uu)$ might vanish on a~nonzero-measure set, which would decrease the discrimination power of $I$.} In addition to the blur invariance, $I$ is also invariant w.r.t. correlation with functions from $\Ss$. This is a~``side-product'' of the assumptions imposed on $\Ss$. The proof of that is the same as before, with only the operations convolution and correlation swapped.

Under the assumptions of the GTBI, $\Ss$ itself is always an equivalence class and $\delta \in \Ss$ (to see this, note that for arbitrary $h \in \Ss $ we have $h = \delta*h = P(\delta*h) = P\delta * h$ which leads to $P\delta = \delta$).

The GTBI is a~very strong theorem because it constructs the blur invariants in a~unified form regardless of the particular class of the blurring PSFs and regardless of the image dimension $d$. The only thing we have to do in a~particular situation is
 to find, for a~given subspace $\Ss$ of the admissible PSFs, an orthogonal projector $P$.
 This is mostly much easier job than to construct blur invariants ``from scratch" for any $\Ss$. This is the most important distinction from our previous paper~\cite{pami:2015}, where the invariants were constructed specifically for $N$-fold symmetric blur without a~possibility of generalization.
 
Before we proceed further, let us show that the assumptions laid on $\Ss$ and $P$ cannot be skipped.

As a~counterexample, let us consider a~1D case where $\Ss$ is a~set of even functions. Let $P$ be defined such that $Pf(x) = f(|x|)$. So, $P$ is a~kind of ``mirroring'' of $f$ and actually, it is a~linear (but not orthogonal) projector on $\Ss$. In this case, $\A$ is a~set of functions that vanish for $x \geq 0$. Clearly, $\A$ is not closed to convolution with even functions and functional $I$ defined in GTBI is not an invariant. 
  
Let us consider another example, again in 1D. Let $\Ss$ be a~set of functions that vanish for any $x < 0$. $\Ss$~is a~linear subspace closed to convolution but it is not closed to correlation. Let us define operator $P$ as follows: $Pf(x) = f(x)$ if $x \geq 0$ and $Pf(x) = 0$ if $x < 0$. Obviously, $P$ is a~linear orthogonal projector onto $\Ss$. However, $\A$ is again not closed to convolution with functions from $\Ss$ and GTBI does not hold. These two simple examples show, that the assumptions of convolution and correlation closure of $\Ss$ and orthogonality of $P$ cannot be generally relaxed (although GTBI may stay valid in some cases even if these assumptions are violated, see Section~\ref{Sec. Examples}).

The property of blur invariance does not say anything about the ability of the invariant to distinguish two different images. In an ideal case, the invariant should be able to distinguish \emph{any} two images belonging to distinct blur-equivalence classes (images sharing the same equivalence class of course cannot be distinguished due to the invariance). Such invariants are called \emph{complete}. The following \emph{completeness theorem} shows that $I$ is a~complete invariant within its definition area.

\begin{theorem}[Completeness theorem] \label{completeness_theorem}
    Let $I$ be the invariant defined by GTBI and let $f, g \in \I \setminus \A$. Then $I(f) = I(g)$ almost everywhere
    if and only if $ f \sim g$. 
\end{theorem}

\begin{proof}
    The proof of the backward implication follows immediately from the blur invariance of $I$. To prove the forward implication, we set $h_1 = Pg$ and $h_2 = Pf$. Then it holds $ f*h_1 = g*h_2 $, which means $ f \sim g$ due to the definition of the equivalence class.
\end{proof}
To summarize, $I$ cannot distinguish functions belonging to the same equivalence class due to the invariance and functions from $\A$ since they do not lie in its definition area. All other functions are fully distinguishable. Note that the completeness may be violated on other image spaces, for instance, on a~space of functions with unlimited support where we find such $f$ and $g$ that $I(f) = I(g)$ at all frequencies where both $I(f)$ and $I(g)$ are well defined but $f$ and $g$ belong to different equivalence classes.

Understanding what properties of $f$ are reflected by $I(f)$ is important both for theoretical considerations as well as for practical application of the invariant. $I(f)$ is a~ratio of two Fourier transforms. As such, it may be interpreted as deconvolution of $f$ with the kernel $Pf$. This ``deconvolution'' eliminates the part of $f$ belonging to $\Ss$ (more precisely, it transfers $Pf$ to $\delta$-function) and effectively acts on the $f_A$ only:
\begin{align*}
    I(f) = & \frac{\F(f)}{\F(Pf)} = \frac{\F(Pf) + \F(f_A)}{\F(Pf)} = 1 + \frac{\F(f_A)}{\F(Pf)}.
\end{align*}

$I(f)$ can be viewed as a~Fourier transform of so-called \emph{primordial image} $f_r$. Even if the primordial image itself may not exist (the existence of $\F^{-1}(I(f))$ is not guaranteed in $\I$), it is a~useful concept that helps to understand how the blur invariants work. The primordial image is unique for each equivalence class, it is the ``most deconvolved'' representative of the class. Two images $f$ and $g$ share the same equivalence class if and only if $f_r = g_r$. For instance, the primordial image of all elements of $\Ss$ is $\delta$-function.

Any element of the equivalence class can be reached from the primordial image through a~convolution. Any features, which describe the primordial image, are unique blur-invariant descriptors of the entire equivalence class. At the same time, the primordial image can also be viewed as a~kind of normalization. It plays the role of a~canonical form of $f$, obtained as the result of the ``maximally possible'' deconvolution of $f$ (see Fig.~\ref{fig:primordial} for schematic illustration).

\begin{figure}[!htbp]
    \centering
    \includegraphics[width=\linewidth]{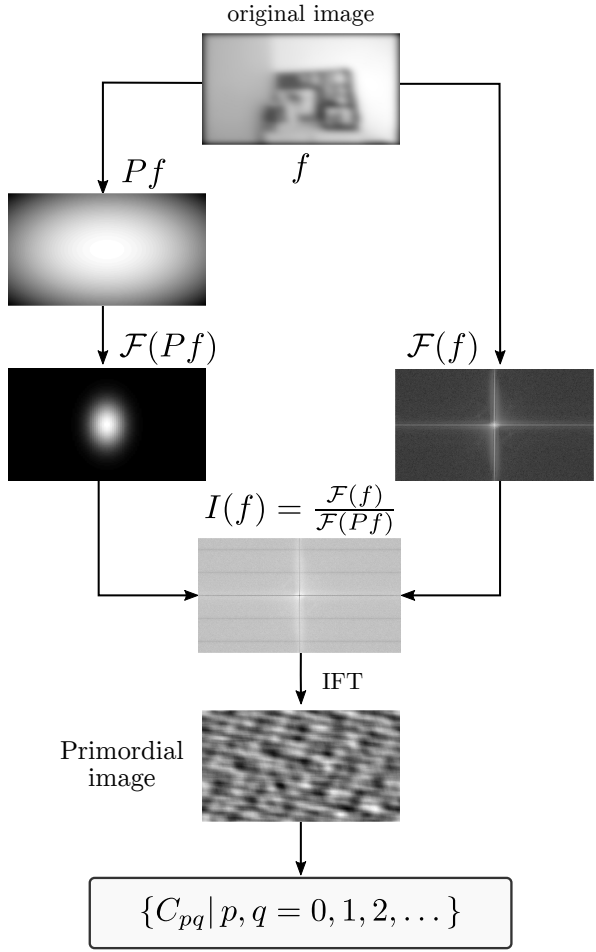}
    \caption{The concept of the primordial image: The blurred image is projected onto $\Ss$ and this projection is used to ``deconvolve" the input image in the Fourier domain. Blur-invariant primordial image is obtained as a~seeming 
    Fourier inversion of $I(f)$. Its moments are blur invariant and can be calculated directly from $f$.}
    \label{fig:primordial}
\end{figure}

As the last topic in this section, we briefly analyze the robustness of $I(f)$ to noise. Let us assume an additive zero-mean white noise, so we have $g = f*h + n$ and, consequently, $Pg = h*Pf + Pn$. As we will see in Section \ref{Sec. Examples}, all meaningful projection operators contain summation/integration over a certain set (often large)  of pixels, which makes $Pn$ to converge to the mean value of $n$, which is zero. So, we have
\begin{eqnarray}
        I(g) \equiv \frac{\F(g)}{\F(Pg) } = 
        \frac{\F(f) \cdot \F(h) + \F(n)}{\F(Pf) \cdot \F(h)} =  \nonumber \\
      =  I(f) + \frac{ \F(n)}{\F(Pf) \cdot \F(h)}.
      \end{eqnarray}
Considering the magnitude of the second term, note that $|\F(n)(\uu)| = \sigma $ because the noise is white. Hence, at least at low frequencies where $\F(Pf) \cdot \F(h)$ dominates, this term is close to zero and $I$ exhibits a robust behavior as $I(g) \doteq I(f)$. However, this may be violated at high frequencies where $\F(Pf) \cdot \F(h)$ is often low.

\section{Invariants and Moments}
\label{Sec. Invariants ans Moments}

The blur invariants defined in the frequency domain by GTBI may suffer from several drawbacks when we use them in practical object recognition tasks. Since $I(f)$ is a~ratio, we possibly divide by very small numbers which requires careful numerical treatment. Moreover, if the input image is noisy, the high-frequency components of $I(f)$ may be significantly corrupted. This can be overcome by suppressing them by a~low-pass filter, but this procedure introduces a~user-defined parameter (the cut-off frequency) which should be set up with respect to the particular noise level. That is why we prefer to work directly in the image domain. Some heuristically discovered image-domain blur invariants were already published in the early papers~\cite{FluSuk:95a, FluSuk:96a,FluSuk:980022}. Here we present a~general theory, which originates from the GTBI.

A straightforward solution might be to calculate an inverse Fourier transform of $I(f)$, which leads to obtaining the primordial image $f_r$ and to characterize $f_r$ by some popular descriptors such as moments. This would, however, be time-consuming and also problematic from the numerical point of view. We would not only have to calculate the projection $Pf$, two forward and one inverse Fourier transforms, but even worse, the result may not lie in $\I$. In this Section, we show how to substantially shorten and simplify this process. We show, that the moments of the primordial image can be calculated directly from the input blurred image, without an explicit construction of $Pf$ and $I(f)$. Since $f_r$ is a~blur invariant, each its moment must be a~blur invariant, too. This direct construction of blur invariants in the image domain, again without specifying particular $\Ss$ and $P$, is the major theoretical result of the paper and performs a~very useful tool for practical image recognition.

Image moments can be defined w.r.t. arbitrary polynomial basis (see Definition~\ref{moments}). In image analysis literature, various bases have been employed to construct moment invariants~\cite{MMIPR}. There is no significant difference among them since between any two polynomial bases there exists a~transition matrix. In other words, from the theoretical point of view, all polynomial bases and all respective moments carry the same information, provide the same recognition power and generate equivalent invariants. However, working with some basis might be in a~particular situation easier than with the others, and also numerical properties and stability of the moments may differ from each other. Here we choose to work with a~basis that \emph{separates} the moments of $Pf$ and $f_A$, although equivalent invariants could be derived in any basis at the expense of the complexity of respective formulas.

Let $\B = \{\pi_{\pp}(\x)\}$ be a~polynomial basis. When considering the polynomials on a~bounded support, then $\B \subset \I$ and all moments $M_{\pp}^{(f)}$ exist and are finite. Let $\Ss$ and $P$ fulfill the assumptions of GTBI. Considering the decomposition $f = Pf + f_A$, we have for the moments
\begin{equation}
    M_{\pp}^{(f)} = M_{\pp}^{(Pf)} + M_{\pp}^{(f_A)} \,.
\end{equation}
We say that $\B$ separates the moments if there exist a~non-empty set of multi-indices $D$ such that it holds for any $f \in \I$
\begin{equation}
    M_{\pp}^{(Pf)} = M_{\pp}^{(f)}
\end{equation}
if $\pp \in D$ and
\begin{equation}
    M_{\pp}^{(Pf)} = 0
\end{equation}
if $\pp \notin D$. In other words, this condition says that the moments are either preserved or vanish under the action of $P$. If fulfilled, the condition also says that the value of $M_{\pp}^{(f_A)}$ is complementary to $M_{\pp}^{(Pf)}$. 

A sufficient condition for $\B$ to separate the moments is that $\pi_{\pp} \in \Ss$ if $\pp \in D$ and $\pi_{\pp} \in \A$ otherwise. Since $\Ss$ and $\A$ are assumed to be mutually orthogonal, the separability of such $\B$ is obvious. This has nothing to do with a~(non)orthogonality of $\B$ itself, as we show in the following simple 1D example. Let $\Ss$ be a~set of even functions and $\A$ be a~set of odd functions. Let $\pi_{p}(x) = x^p$. If we take $D = \{p=2k|k \geq 0\}$, we obtain the moment-separating polynomials.

For the given $\Ss$ and projector $P$, the existence of a~basis that separates the moments is not guaranteed, although in most cases of practical interest we can find some. If it does not exist, the moment blur invariants still can be derived. It is sufficient if the moments $M_{\pp}^{(Pf)}$ can be expressed in terms of $M_{\pp}^{(f)}$ if $\pp \in D$ and some functions of $M_{\pp}^{(f)}$ equal zero for $\pp \notin D$. This makes the derivation more laborious and the formulas more complicated but does not make a~principle difference. Anyway, to keep things simple, we try for any particular $\Ss$ to find such $\B$ that provides the moment separability.

To get the link between $I(f)$ and the moments $M_{\pp}^{(f)}$, we recall that Taylor expansion of Fourier transform is
\begin{equation} \label{expansion1}
    \F(f)(\uu) = \sum_{\pp} \frac{(-2 \pi i)^{|\pp|}}{\pp!} m_{\pp}^{(f)} \uu^{\pp}
\end{equation}
where $m_{\pp}$ is a~geometric moment. In the sequel, we assume that the power basis $\pi_{\pp}(\x) = \x^{\pp}$ separates the moments. If it was not the case, one would substitute into~\eqref{expansion1} any separating basis through the polynomial transition relation.

The GTBI can be rewritten as
\begin{equation}
    \F(Pf)(\uu) \cdot I(f)(\uu) = \F(f)(\uu)\,.
\end{equation}
All these three Fourier transforms can be expanded similarly to~\eqref{expansion1} into absolutely convergent Taylor series. Thanks to the moment separability, we can for any ${\pp} \in D$ simply write $m_{\pp}^{(Pf)} = m_{\pp}^{(f)} = m_{\pp}$.
So, we have 

\begin{align}
    \sum_{\pp \in D} \frac{(-2 \pi i)^{|\pp|}}{\pp!} m_{\pp}^{(f)} \uu^{\pp} \cdot
    \sum_{\pp} \frac{(-2 \pi i)^{|\pp|}}{\pp!} C_{\pp} \uu^{\pp} = \notag    \\
    \sum_{\pp} \frac{(-2 \pi i)^{|\pp|}}{\pp!} m_{\pp}^{(f)} \uu^{\pp} \,,
\end{align}
where $C_{\pp}$ can be understood as the moments of the primordial image $f_r$. Comparing the coefficients of the same powers of $\uu$ we obtain, for any $\pp$
\begin{align}
    \sum_{\kk \in D}^{\pp} \frac{(-2 \pi i)^{|\kk|}}{\kk!} \frac{(-2 \pi i)^{|\pp-\kk|}}{(\pp-\kk)!} m_{\kk} C_{\pp-\kk} = \notag \\
    \frac{(-2 \pi i)^{|\pp|}}{\pp!}m_{\pp} \,,
\end{align}
which can be read as
\begin{equation}
    \sum_{\kk \in D}^{\pp} \binom{\pp}{\kk} m_{\kk} C_{\pp-\kk} = m_{\pp} \,.
\end{equation}
The summation goes over those $\kk \in D$ for which $0 \leq k_i \leq p_i, \ i=1, \ldots,d$. Note that always ${\bf 0} \in D$. (To see that, it is sufficient to find an image whose zero-order moment is preserved under the projection. Such an example is $\delta$-function, because $P(\delta) = \delta$, as we already showed.)

After isolating $C_{\pp}$ on the left-hand side we obtain the final recurrence
\begin{equation} \label{blur_inv_general2}
    m_{\mathbf{0}}C_{\pp} = m_{\pp} - \sum_{\substack{\kk \in D\\ \kk \neq \mathbf{0}}}^{\pp} \binom{\pp}{\kk} m_{\kk} C_{\pp-\kk} \,.
\end{equation}

This recurrence formula is a~general definition of blur invariants in the image domain (provided that $m_{\bf 0} \neq 0$)\footnote{If $m_{\bf 0} = 0$, then $C_{\pp}$ is not defined. We find the first non-zero moment $m_\nn, \nn \in D$ and derive an analogous recurrence for $C_{\pp-\nn}$.}. Since $I(f)$ has been proven to be invariant to blur belonging to $\Ss$, all coefficients $C_{\pp}$ must also be blur invariants. The beauty of Eq.~\eqref{blur_inv_general2} lies in the fact that we can calculate the invariants from the moments of $f$, without constructing the primordial image explicitly either in frequency or in the spatial domain.


Some of the invariants $C_{\pp}$ are trivial for any $f$ and useless for recognition. We always have $C_{\mathbf{0}} = 1$ and some other invariants may be constrained depending on the index set $D$. If for arbitrary ${\pp}, \kk \in D$ also $(\pp - \kk) \in D$, then $C_{\pp} = 0$ for any ${\pp} \in D$ as can be deduced from Eq.~\eqref{blur_inv_general2} via induction. This commonly happens in many particular cases of practical interest and then only the invariants with ${\pp} \notin D$ should be used. In addition to that, some invariants may vanish depending on $f$. In particular, if $f \in \Ss$, then $C_{\pp} = 0$
for \emph{any} ${\pp} \neq \mathbf{0}$.

Numerical behavior of one particular moment invariant of the type (\ref{blur_inv_general2}) of order 7 can be seen in Fig. 
 \ref{fig:MRE_I_43}, where the mean relative error (MRE) between the invariant of the blurred and noisy image and the original one is depicted as a function of the blur size and SNR. Note that the MRE almost does not depend on the blur size (since the blur was synthetic, we eliminated the boundary effect), is below 0.2\% if the noise is mild and even for heavy noise of SNR = 10 the MRE is still below 1\%, which shows an excellent robustness.
 The behavior of other invariants is similar. However, when increasing the order of the moments used, the MRE slightly increases as well. Summarizing, the robustness to noise is determined by the robustness of the moments, which has been thoroughly studied in many papers (see \cite{2D3D} and further references thereof) and is known to be quite good.
 
 \begin{figure}
    \centering
    \includegraphics[width = \linewidth]{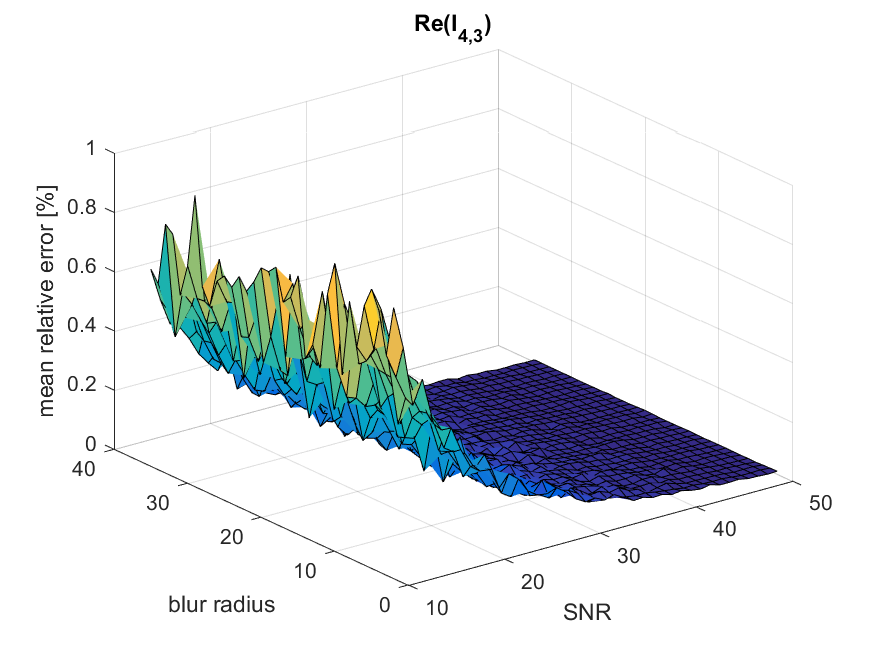}
    \caption{The MRE between the invariant of the blurred and noisy image and the original one  as a function of the blur size and SNR.}
    \label{fig:MRE_I_43}
\end{figure}
 
\section{Blur Examples}
\label{Sec. Examples}

In this Section, we show the blur invariants provided by the GTBI for several concrete choices of $\Ss$ and $P$ with a~particular focus on those of practical importance in image recognition. Some of them are equivalent to the invariants already published in earlier papers; in such cases, we show the link between them. Some other invariants are published here for the first time.

\subsection{Trivial cases}

The formally simplest case ever is $\Ss = \I$ and $Pf = f$. Although this choice fulfills the assumptions of GTBI, it is not of practical importance because the entire image space forms a~single equivalence class, and any two images are blur equivalent. Actually, GTBI yields $I(f) = 1$ for any $f$.

An opposite extreme is to choose $\Ss = \{a \delta| a \in \R \}$. This ``blur'' is in fact only a~contrast stretching. If we set $Pf = (\int f) \cdot \delta$, $P$ is not orthogonal but still $P({f*h}) = Pf*h$ and GTBI can be applied provided that $\int f \neq 0$. We obtain $I(f) = \F(f) /\int f$, which leads to a~contrast-normalized primordial image $f_r = f/\int f$.

Another rather trivial case is $\Ss = \left\{h\left|\int \right. h =1 \right\}$. This is the set of all brightness-preserving blurs without any additional constraints. We may construct $Pf = f/\int f$, which actually is a~projector; however it is neither linear nor orthogonal. Since $P(f*h) = Pf*h$, we can still apply GTBI, which yields  a~single-valued blur invariant $I(f) = \int f$, that corresponds to the primordial image $f_r = (\int f) \cdot \delta$.


\subsection{Symmetric blur in 1D}

In 1D, the only blur space $\Ss$, which can be defined generically and is of practical interest, is the space of all even functions. 1D symmetric blur invariants were firstly described in~\cite{FluSuk:97a} and later adapted to wavelet domain by Makaremi~\cite{makaremi:wvl}. Kautsky~\cite{KauFlu:11} rigorously investigated these invariants and showed how to construct them in terms of arbitrary moments. Galigekere~\cite{gali:radon} studied the blur invariants of 2D images in the Radon domain, which inherently led to 1D blur invariants.

If we consider the projector 
\begin{equation}
    Pf(x) = (f(x) + f(-x))/2
\end{equation}
 then $\A$ is a~space of odd functions, $P$ is orthogonal and GTBI can be applied directly. As for the moment expansion, the simplest solution is to use the standard monomials $\pi_p(x) = x^p$,
which separate the geometric moments for $D$ being the set of even non-negative indices.

\subsection{Centrosymmetric blur in 2D}
\label{Centrosymmetric_blur_in_2D}

Invariants w.r.t. centrosymmetric blur in 2D have attracted the attention of the majority of authors who have been involved in studying blur invariants. The number of papers on this kind of blur exceeds significantly the number of all other papers on this field. This is basically for two reasons -- such kind of blur appears often in practice and the invariants are easy to find heuristically, without the knowledge of the state-of-the-art theory of projection operators.

A natural way of defining $P$ is
\begin{equation}
    Pf(x,y) = (f(x,y) + f(-x,-y))/2.
\end{equation}
Then standard geometric moments are separated at $D = \{(p,q)|(p+q) \text{ even}\}$ and Eq.~\eqref{blur_inv_general2} leads to moment expansion that appeared in some earlier papers such as in~\cite{FluSuk:980022} and others cited in Section~\ref{State_of_the_art_of_blur_invariants}. 

This approach can be extended into 3D, where the definition of centrosymmetry is analogous. Existing 3D blur invariants~\cite{FluBolZit, FluBol:pami03} are just special cases of Eq.~\eqref{blur_inv_general2}.

\subsection{Radially symmetric blur}
\label{Radially_symmetric_blur}

Radially (circularly) symmetric PSF’s satisfying $h(r,\phi) = h(r)$ appear in imaging namely as an out-of-focus blur on a~circular aperture (see Fig.~\ref{fig:polygonal_apertures} (a) for an example). The projector $P_\infty$ is defined as
\begin{equation}
    (P_\infty f)(r) = \frac{1}{2 \pi } \int\limits_{0}^{2 \pi} f(r, \phi) \ud \phi \,.
\end{equation} 

The standard power basis does not separate the moments. This is why various radial moments have been used to ensure the separation. Basis $\B$ consists of circular harmonics-like functions of the form $\pi(r,\phi) = R_{pq}(r)\mathrm{e}^{i\chi(p,q)\phi}$, where $R_{pq}(r)$ is a~radial polynomial and $\chi(p,q)$ is a~simple function of the indices. There are several choices of $\B$, which separate the respective moments and yield blur invariants (the index set $D$ depends on the particular $\B$). Some of them were introduced even without the use of projection operators. They mostly employed Zernike moments~\cite{zhu:ZernikePAA, beijing:Zernike, hanjie:Zernike, xiubin:pseudo_Zernike}, Fourier-Mellin moments~\cite{FMM} and complex moments~\cite{FluZit:icpr04}.

\subsection{$N$-fold symmetric blur}

$N$-fold rotationally symmetric blur performs one of the most interesting cases, both from theoretical and practical points of view. This kind of blur appears as an out-of-focus blur on a~polygonal aperture. Most cameras have an aperture the size of which is controlled by physical diaphragm blades, which leads to polygonal or close-to-polygonal aperture shapes if the diaphragm is not fully open (see Fig.~\ref{fig:polygonal_apertures} (b) and (c)).

\begin{figure}[htbp]
    \centering
    \begin{subfigure}{0.23\linewidth}
        \includegraphics[width=\linewidth]{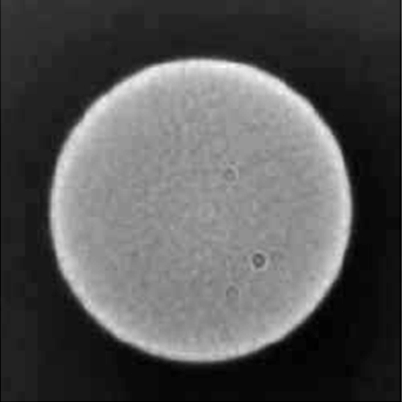}
        \caption{}
    \end{subfigure}
    \begin{subfigure}{0.23\linewidth}
        \includegraphics[width=\linewidth]{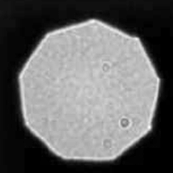}
        \caption{}
    \end{subfigure}
    \begin{subfigure}{0.23\linewidth}
        \includegraphics[width=\linewidth]{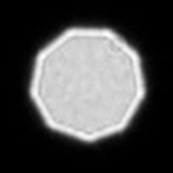}
        \caption{}
    \end{subfigure}
    \begin{subfigure}{0.23\linewidth}
        \includegraphics[width=\linewidth]{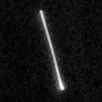}
        \caption{}
    \end{subfigure}
    \caption{Blurring PSFs obtained as photographs of a~bright point. Out-of-focus blur on 
    a~circular aperture (a), on polygonal apertures (b), (c), and directional blur (d). Aperture (a) has radial symmetry, (b) has 9-fold rotation symmetry and (c) exhibits {9-fold} dihedral symmetry.}
    \label{fig:polygonal_apertures}
\end{figure}

The blur space is defined as
\begin{equation}
    \Ss_{N} = \{h | h(r,\theta) = h(r,\theta + 2\pi/N)\} \,.
\end{equation}
$\Ss_{N}$ is a~vector space closed under convolution and correlation. We can construct projector $P_N$ as
\begin{equation}
    (P_Nf)(r, \theta) = \frac{1}{N}\sum_{j=1}^{N} f(r, \theta + \alpha_j) \,,
\end{equation} 
where $\alpha_j = 2 \pi j/N$. Since $P_N$ is an orthogonal projector, GTBI can be immediately applied. Complex moments are separated with
\begin{equation}
    D = \{(p,q)| (p-q)/N \text{ is integer} \}
\end{equation}
which allows to get particular blur invariants from Eq.~\eqref{blur_inv_general2}.

Invariants to $N$-fold symmetric blur were originally studied in~\cite{pami:2015}, where the idea of projection operators appeared for the first time. Their application for registration of blurred images was reported in~\cite{matteo:tip}.

\subsection{Dihedral blur }

The $N$-fold symmetry, discussed in the previous subsection, may be coupled with the axial symmetry. In such a~case, the number of the axes equals $N$ and we speak about the $N$-fold dihedral symmetry. Many out-of-focus blur PSFs are actually dihedral, particularly if the diaphragm blades are straight (see Fig.~\ref{fig:polygonal_apertures} (c)).

The blur space $\mathcal{D}_N$ is a~subset of $\Ss_N$ given as
\begin{equation}
    \mathcal{D}_N = \{h \in \Ss_N|\, \exists \alpha \text{ such that } h(x,y) = h^\alpha(x,y) \} \,,
\end{equation}
where ${\alpha} \in \langle 0, \pi/2 \rangle $ is the angle between the symmetry axis~$a$ and the $x$-axis and $h^\alpha(x,y)$ denotes function $h(x,y)$ flipped over~$a$.
However, the set $\mathcal{D}_N$ is \emph{not} closed under convolution if we allow various axis directions. Only if we fix the symmetry axis orientation to a~constant angle ~$\alpha$,
we get the closure property. Then we can define the projection operator $Q_{N}^{\alpha}$ as
\begin{equation}
    Q_{N}^{\alpha}f = P_N(f + f^{\alpha})/2
\end{equation}
and GTBI can be applied.

Dihedral blur invariants were firstly studied in~\cite{boldys:dihedral}. Their major limitation comes from the fact that the orientation of the symmetry axis must be apriori known (and the same for all images entering the classifier). This is far from being realistic and the only possibility is to estimate $\alpha$ from the blurred image itself~\cite{matteo:tip2015}.

\subsection{Directional blur}

Directional blur (sometimes  called linear motion blur) is a~2D blur of a~1D nature that acts in a~constant direction only. 
Directional blur may be caused by camera shake, scene vibrations, and camera or scene  motion. The velocity of the motion may vary during the acquisition, but this model assumes the motion along the line. We do not consider a general motion blur along an arbitrary curve in this paper.\footnote{Imposing no restrictions on the blur trajectory would lead to a very broad blur space, where only trivial invariants exist.} 

The respective PSF has the form (for the sake of simplicity, we start with the horizontal direction)
\begin{equation} \label{direction_blur}
    h(x,y) = h_1(x)\delta(y) \,,
\end{equation}
where $h_1(x)$ is an arbitrary 1D image function. The space $\Ss$ is defined as a~set of all functions of the form~\eqref{direction_blur}. When considering a~constant direction only, $\Ss$~is closed under 2D convolution and correlation. The projection operator $P$ is defined as
\begin{equation}
\label{Pdirectional}
    Pf(x,y) = \delta(y) \int f(x,y) \ud y \,.
\end{equation}
$P$ is not orthogonal but geometric moments are separated with $D = \{(p,q)|\, q = 0 \}$ and Eq.~\eqref{blur_inv_general2} yields the directional blur invariants in terms of geometric moments.



If the blur direction under a~constant angle $\beta$ is known, the projector $P^\beta f$ is defined analogously to~\eqref{Pdirectional} by means a~line integral along a~line which is perpendicular to the blur direction (see Fig.~\ref{fig:polygonal_apertures} (d) for an example of a~real directional PSF).



The idea of  invariants to linear motion blur appeared for the first time in~\cite{FluSuk:96} and in a~similar form in~\cite{stern}, without any connection to the projection operator. Zhong used the motion blur invariants for recognition of reflections on a~waved water surface~\cite{water:TIP}. Peng et al. used them for weed recognition from a~camera moving quickly above the field~\cite{weed} and for classification of wood slices on a~moving conveyor belt~\cite{woodslice} (these applications were later enhanced by Flusser et al.~\cite{woodslice-reaction, weed-reaction}). Other applications can be found in~\cite{guan-biologically, wang-sinusoidal}. The necessity of knowing the blur direction beforehand is, however, an obstacle to the wider usage of these invariants.

\subsection{Gaussian blur}

Gaussian blur appears whenever the image has been acquired through a~turbulent medium.
It is also introduced into the images as the sensor blur due to the finite size of the sampling pulse and may be sometimes applied intentionally as a~part of denoising.


Since Gaussian function has an unlimited support, we have to extend our current definition of $\I$ by including functions of exponential decay. We define the set $\Ss$ as
\begin{equation} \label{spaceS}
    \Ss = \{aG_\Sigma|\, a >0, \Sigma \text{ positive definite}\} \,,
\end{equation}
where $\Sigma$ is the covariance matrix which controls the shape of the Gaussian~$G_\Sigma$.

$\Ss$ is closed under convolution but it is not a~vector space. We define $Pf$ to be such element of $\Ss$ which has the same integral and covariance matrix as the image~$f$ itself. Clearly, $P^2 = P$ but $P$ is neither linear nor orthogonal. Although the assumptions of GTBI are violated, the Theorem still holds thanks to $P(f*h) = Pf*h$. The moment expansion analogous to Eq.~\eqref{blur_inv_general2} can be obtained when employing the parametric shape of the blurring function. Thanks to this, we express all moments of order higher than two as functions of the low-order ones, which substantially increases the number of non-trivial invariants.


Several heuristically found Gaussian blur moment invariants appeared in~\cite{tianxu:gauss, xiao-gauss, gauss-metric-conf, gauss-metric}. Invariants based on projection operators were proposed originally in~\cite{Gauss-TIP2015} 
for circular Gaussians and in~\cite{PR2019:gauss} for blurs with a~non-diagonal covariance matrix.



\section{Experimental Evaluation}

In this section, we show the performance of the proposed invariants in the recognition of blurred facial photographs, in template matching within a blurred scene and in two common image processing problems
-- multichannel deconvolution and multifocus fusion -- where we use the proposed invariants  for registration of blurred frames. The first  experiment was performed on simulated data, which makes possible to evaluate the results quantitatively, while the other three experiments show the performance on real images and blurs.

\subsection{Face recognition}
The use of various CNNs for recognition of blurred images has been tested recently in several papers, that studied the impact of blur on the network recognition performance~\cite{vasiljevic2016examining, zhou2017classification, dodge2016understanding, pei_CNNdegradationPAMI}. They all reported that introducing even a~small or moderate blur decreases the performance of networks trained on clear images only. Some of the above papers recommended eliminating this drawback by network fine-tuning or by augmentation of the training set with many blurred versions of the training images, however at the expense of a~massive increase of the training time.
   
We used 38 facial images of distinct persons from the YaleB dataset~\cite{GeBeKr01} (frontal views only). Each class was represented by a~single image resized to $256\times256$ pixels and normalized to brightness. As the test images, we used synthetically blurred and noisy instances of the database images starting from mild  ($5 \times 5$ blur, SNR = 50 dB) to heavy ($125 \times 125$ blur, SNR = 5 dB) distortions. We used four types of centrosymmetric blur (circular, random, linear motion, Gaussian) and Gaussian white noise in these simulations (see Fig.~\ref{fig:face:blurred:2} for some examples). In each setting, we generated 10 instances of each database image. 

The faces were classified by four different methods -- blur invariants, CNN trained on clear images only, CNN trained on images  augmented with blur, and 
the Gopalan's distance~\cite{gopalan}.
 As blur invariants, we used the particular version of $I(f)$ from Theorem 6 with operator $P$ defined in Section~\ref{Centrosymmetric_blur_in_2D}. As the CNN, we used a~pre-trained ResNet18~\cite{He_2016_CVPR} initially trained on the ImageNet dataset~\cite{imagenet_cvpr09}. Data augmentation was done by adding 100 differently blurred and noisy instances to the training set such that the blur was of the same size as that of the test images. The Gopalan's distance 
belongs to ``handcrafted" features  and measure the "distance" between two images in a~way that should be insensitive to blur.  Unlike the proposed invariants, the Gopalan's method  requires the knowledge of the blur support size, which is no problem in simulated experiments.

The recognition results are summarized in Table~\ref{tab:nn:blurs}.
The performance of the proposed invariants is excellent except for the last two settings, where the blur caused extreme smoothing and significant boundary effect (but still the performance over 90\% is very good). Fig.~\ref{fig:face_g_fail} shows examples of a~very heavy blur that was handled correctly by the proposed invariants. The CNN trained on clear images only fails for mid-size and large blurs, which corresponds to the results of earlier studies. However, if we augment the training data extensively with blurred images, the performance is close to 100\% but the training time was about four hours compared to few seconds required by the invariants. In this scenario, introducing new images/persons to the database requires additional lengthy training of CNNs.
The performance of the Gopalan's method decreases as the blur increases because this method is blur-invariant only approximately. Its computing complexity is less than that of the augmented CNN but much higher than that of the proposed invariants and the CNN without augmentation. 


\begin{table}[!t]
    \centering
    \begin{tabular}{r|c||c|c|c|c}
        \multicolumn{2}{c||}{Degradation}	& \multicolumn{4}{c}{Method} \\
        \hline
        SNR    &  Blur & In    & CNN	    & A-CNN 	& G	\\
        \hline  
        \hline
        50 & circular \hfill 5x5    & 100     &	100	& 100 & 100  	\\
        50 & circular \hfill 10x10    & 100	    &	90	& 100 & 98		\\
        50 & circular \hfill 15x15    & 100     &	35	& 100 & 76	  	\\
        50 & circular \hfill 125x125    & 100     &	-	& 100 & 40	\\         	
        \hline
        5 &  circular \hfill 125x125    & 99,9     &	-	& 99,8  & 5 	\\
        5 &  random \hfill 125x125    & 99,8	    &	-	& 99,9 	& 5	\\
        5 &  motion \hfill 125x125    & 92	    &	-	& 99,6 	& 3	\\
        5 &  Gaussian \hfill 125x125    & 91	    &	-	& 99,5 	& 3	
    \end{tabular} 
    \caption{The recognition rate~[\%] for different degradations
    achieved by the proposed invariants (In), CNN trained on clear images, CNN with augmentation by blurred images, and the Gopalan's method \cite{gopalan} (G).
}
    \label{tab:nn:blurs}
\end{table}


\begin{figure}[htb]
    \centering
    \includegraphics[width = 0.15\linewidth]{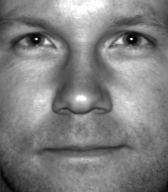}
	\includegraphics[width = 0.15\linewidth]{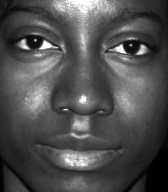}
	\includegraphics[width = 0.15\linewidth]{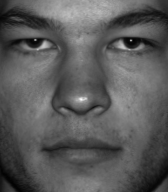}
	\includegraphics[width = 0.15\linewidth]{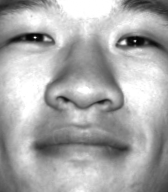}
	\includegraphics[width = 0.15\linewidth]{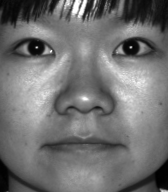}
	\includegraphics[width = 0.15\linewidth]{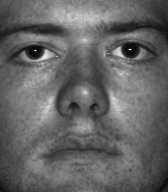}
	
    \includegraphics[width = 0.15\linewidth]{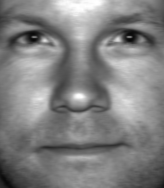}
    \includegraphics[width = 0.15\linewidth]{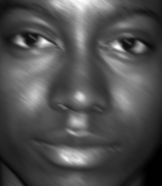}
    \includegraphics[width = 0.15\linewidth]{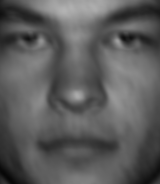}
    \includegraphics[width = 0.15\linewidth]{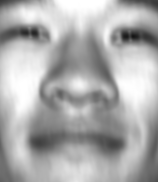}
    \includegraphics[width = 0.15\linewidth]{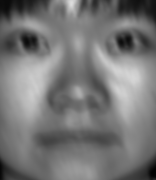}
    \includegraphics[width = 0.15\linewidth]{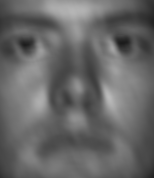}
	
    \includegraphics[width = 0.15\linewidth]{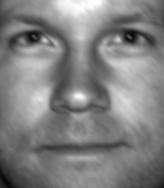}
    \includegraphics[width = 0.15\linewidth]{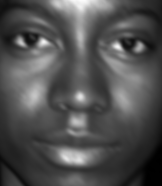}
    \includegraphics[width = 0.15\linewidth]{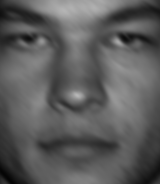}
    \includegraphics[width = 0.15\linewidth]{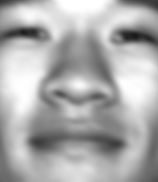}
    \includegraphics[width = 0.15\linewidth]{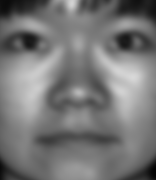}
    \includegraphics[width = 0.15\linewidth]{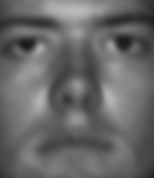}
	
    \includegraphics[width = 0.15\linewidth]{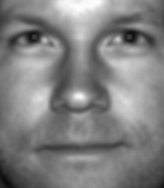}
    \includegraphics[width = 0.15\linewidth]{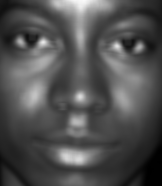}
    \includegraphics[width = 0.15\linewidth]{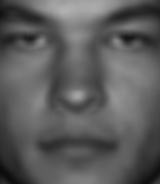}
    \includegraphics[width = 0.15\linewidth]{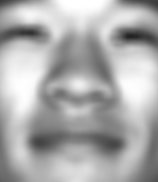}
    \includegraphics[width = 0.15\linewidth]{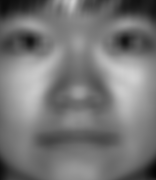}
    \includegraphics[width = 0.15\linewidth]{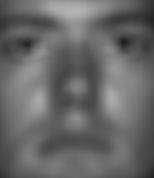}
	
    \includegraphics[width = 0.15\linewidth]{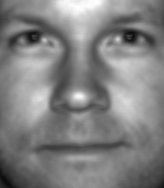}
	\includegraphics[width = 0.15\linewidth]{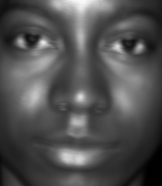}
    \includegraphics[width = 0.15\linewidth]{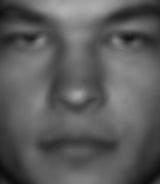}
    \includegraphics[width = 0.15\linewidth]{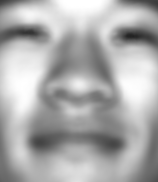}
    \includegraphics[width = 0.15\linewidth]{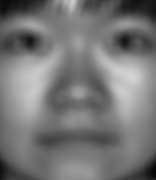}
    \includegraphics[width = 0.15\linewidth]{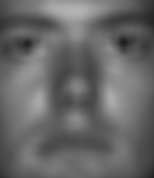}
    \caption{
    Sample faces used in the experiment. From top to bottom: no blur, motion, Gaussian, uniform and random blur; from left to right: 
    blur size 5, 7, 9, 11, 13, and 15 pixels.
    }
\label{fig:face:blurred:2}
\end{figure}

\begin{figure}[!t]
    \centering
    \includegraphics[width=0.28\linewidth]{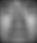}
    \includegraphics[width=0.28\linewidth]{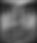}
    \includegraphics[width=0.28\linewidth]{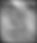}
    \includegraphics[width=0.28\linewidth]{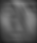}
    \includegraphics[width=0.28\linewidth]{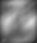}
    \includegraphics[width=0.28\linewidth]{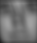}
    \caption{Extreme cases recognized correctly by the blur invariants but misclassified both by CNN and the Gopalan's method. }
    \label{fig:face_g_fail}
\end{figure}

\subsection{Template matching}

Localization of sharp templates in a~blurred scene is a~common task in many application areas such as in landmark-based image registration and in stereo matching. In this experiment, we show how the blur invariants can be used for this purpose.

We took two pictures of the same indoor scene -- the first one was sharp while the other one was intentionally taken with wrong focus. In the sharp image, we selected 21 square templates (see Fig.~\ref{fig:TM}a) and the goal was to find these templates in the blurred scene. Since the out-of-focus blur has approximately a~circular shape, we used the blur invariants w.r.t. radially symmetric blur (see Section~\ref{Radially_symmetric_blur}).
Since the templates are relatively small, we used the invariants defined directly in the image domain by means of moments (\ref{blur_inv_general2}).
The matching was performed by searching over the
whole scene, without using any prior information about the template
position. The matching criterion was the minimum distance
in the space of blur invariants. Nine templates were localized with an error less than or equal to 10 pixels, eight templates with an error 11 -- 20 pixels, three templates with an error 21 -- 30 pixels, and one template with an error greater than 30 pixels (see Fig.~\ref{fig:TM}b). In the sense of a~target error, each template was localized in a~position that is less than half of the template size from the ground truth. 

The localization error is caused by the fact that the blurred template is not exactly a~convolution of the ground truth template and the PSF. We observe a~strong boundary effect as pixels outside the template influence pixels inside the template. This interaction is, of course, beyond the assumed convolution model. In the case of a~large PSF, it influences the matching. If the distance matrix has a~flat minimum, then a~small disturbance of the invariants due to the boundary effect may result in an inaccurate match. 

For comparison, we performed the same task using plain moments instead of the invariants while keeping the number and order of the features the same. Results are unacceptable, most of the templates were matched in totally wrong positions (see
 Fig.~\ref{fig:TM}c). This clearly shows that introducing blur-invariant features brings a~significant improvement.

\begin{figure}[!ht]
    \centering
    \includegraphics[width = 0.94\linewidth]{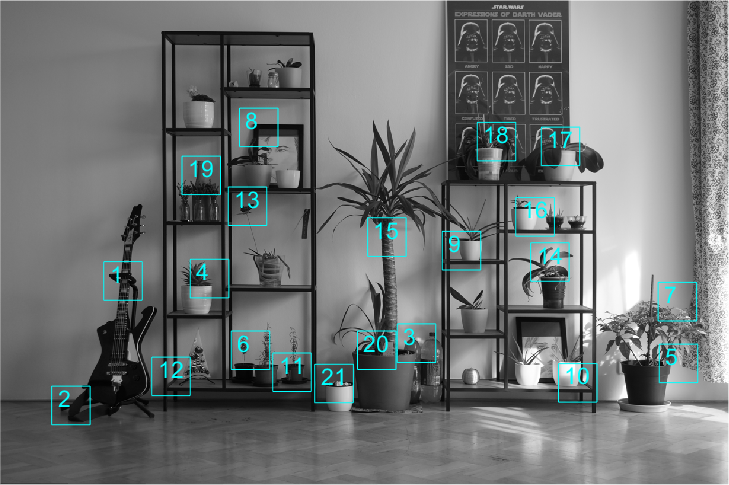}
    (a) 
     \includegraphics[width = .94\linewidth]{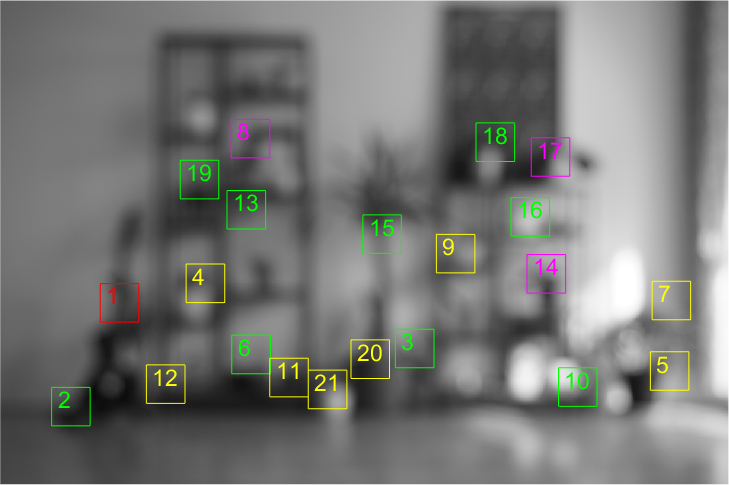}
    (b)\vspace{1ex}
    \includegraphics[width = .94\linewidth]{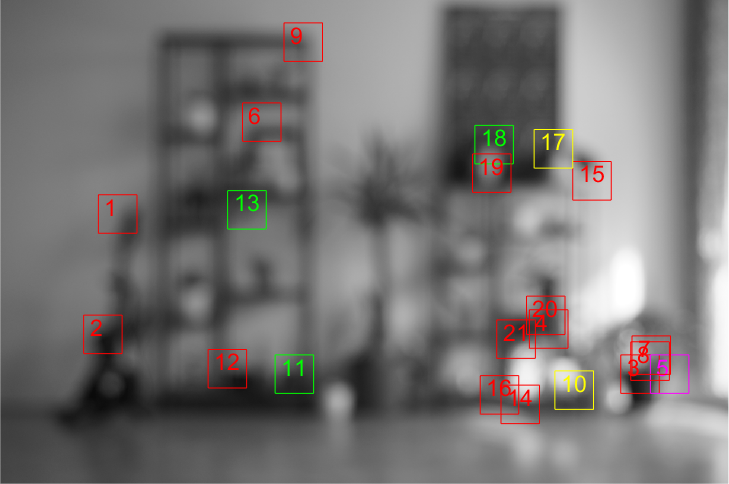}
    (c) \vspace{1ex}
    \caption{Template matching experiment. The original sharp scene with the selected templates (a), the matched templates in the defocused scene using blur invariants (b) and the same using plain moments (c). The template color encodes the localization error (green: 0-10 pixels, yellow: 11-20, violet: 21-30, red: $>30$). }
    \label{fig:TM}
\end{figure}


\subsection{Multichannel deconvolution}

Multichannel blind deconvolution (MBD) is a process where two or more differently blurred images of the same scene are given as an input and a single de-blurred image is obtained as an output \cite{hnedakniha}. The restoration is blind, so no parametric form of the PSF's is required. Comparing to single-channel deconvolution, it is more stable and usually produces much better results. However, the crucial requirement is that the input frames must be registered before entering the deconvolution procedure. The registration accuracy up to several pixels is sufficient because advanced MBD algorithms are able to compensate for a small misalignment 
\cite{SroFlu:05}. 
Since the input frames are blurred, most of the common registration techniques designed originally for sharp images 
\cite{ZitFlu:survey} fail.

For the registration of blurred frames, the proposed invariants can be used. 
In Fig. \ref{fig:MBD} (left and middle), we see two input images of a statue blurred by camera shake. Since the camera was handheld and there was a few-second interval between the acquisitions, the images differ from each other not only by the particular blur but also by a shift and a small rotation. To register them, we use "blur-invariant phase correlation" method. It is an efficient landmark-free technique inspired by traditional phase correlation \cite{castro}. Our method uses directly the blur invariants $I(f)$ and $I(g)$ (instead of whitened Fourier spectrum $F/|F|$ and $G/|G|$ used in the phase correlation) to find the correlation peak. 
Since we do not have much prior information about the blurs, we use operator $P_2$ from Section V.E to design the invariants, because it is less specific than the others and should work for many blurs.
Switching between Cartesian and polar domains, the method can register both shift and rotation.

In this real-data example we do not have any ground truth so we cannot explicitly measure the registration accuracy. However, it is documented by a good performance of the subsequent MBD algorithm. The registered frames were used as an input of the MBD proposed in
\cite{Kotera17}. The result can be seen in Fig. \ref{fig:MBD} right. We acknowledge a sharp image with very little artifacts, which proves a sufficient registration accuracy (and of course a good performance of the MBD algorithm itself). 

\begin{figure*}
    \centering
    \includegraphics[width = 0.32\textwidth]{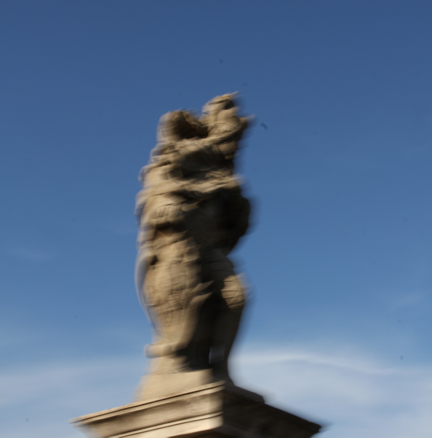}
    \includegraphics[width = 0.32\textwidth]{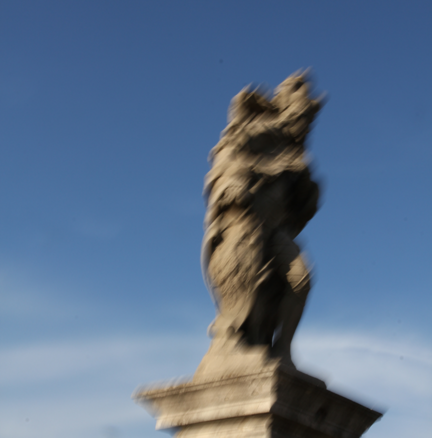}
    \includegraphics[width = 0.285\textwidth]{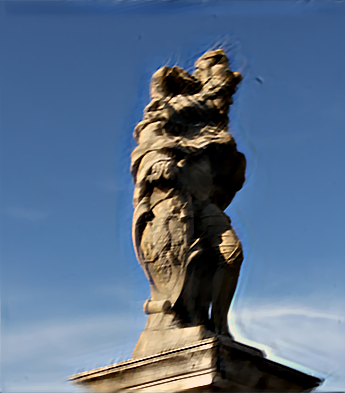}
    \caption{Multichannel deconvolution. The 
    original frames blurred by a camera shake (left and middle). Note the shift and rotation misalignment between them, that was registered by blur-invariant phase correlation.
     The  result of MBD \cite{Kotera17} applied on the registered frames (right).} 
    \label{fig:MBD}
\end{figure*}

\subsection{Multifocus  fusion}

Multifocus image fusion (MIF) is a well-known technique of combining two or more images of the same 3D scene, that were taken by a camera with a shallow depth of field \cite{modrakniha}. Typically, one frame is focused to the foreground while the other one to the background (see Fig. \ref{fig:MFF} for an example). The fusion algorithms basically decide locally in which frame this part of the scene is best focused and generate the fused image by stitching the selected parts together without performing any deconvolution.
Obviously, an accurate registration of the inputs is a key requirement.

The registration problem is here even more challenging than in the previous experiment, because the convolution model holds only on the foreground or background and the required accuracy is higher them in the MBD case. 

The input frames and the fused product are shown in Fig. \ref{fig:MFF}. As in the previous experiment, we applied the blur-invariant phase correlation. Since there was just a shift between the frames, the entire procedure run in the Cartesian coordinates. We assumed a circular out-of-focus blur, so we used the operator $P_\infty$ from Section V.D. After the registration, the fusion itself was performed by the method 
proposed in
\cite{ZHANG202099}. High visual quality  of the fused product with almost no artifacts proves the accuracy of the registration.

\begin{figure*}
    \centering
    \includegraphics[width = 0.32\textwidth]{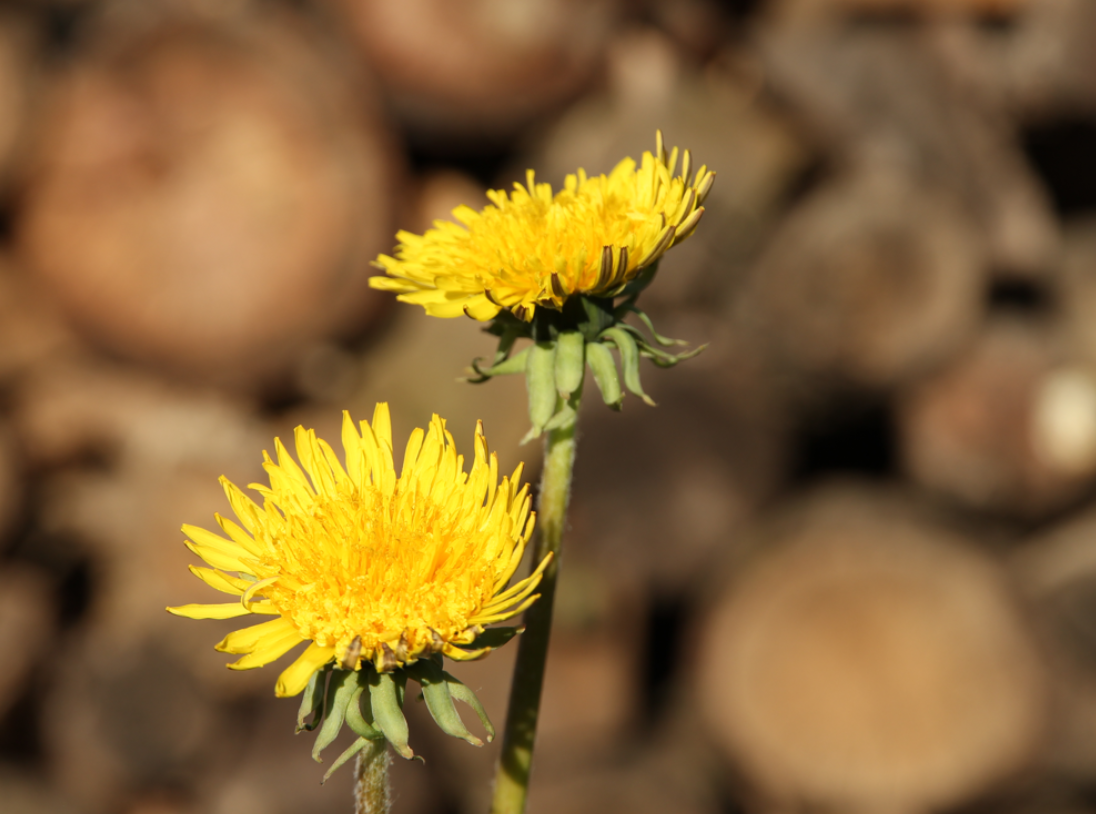}
    \includegraphics[width = 0.32\textwidth]{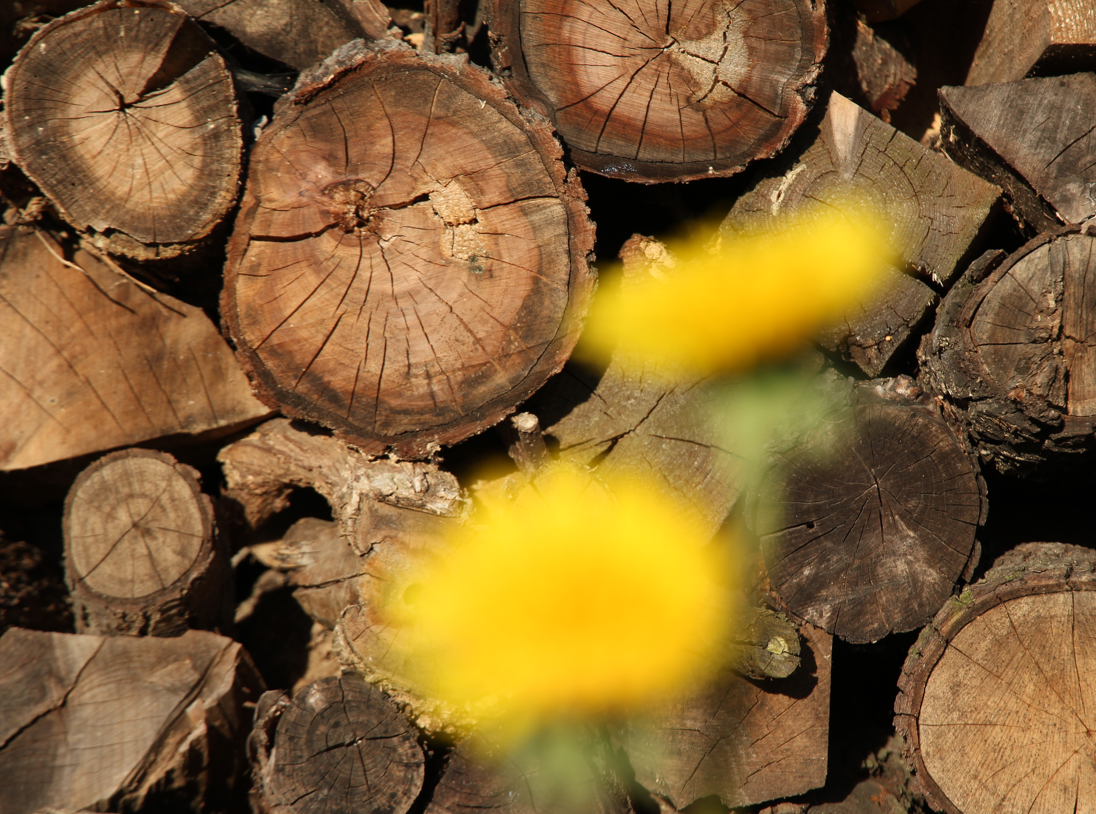}
    \includegraphics[width = 0.28\textwidth]{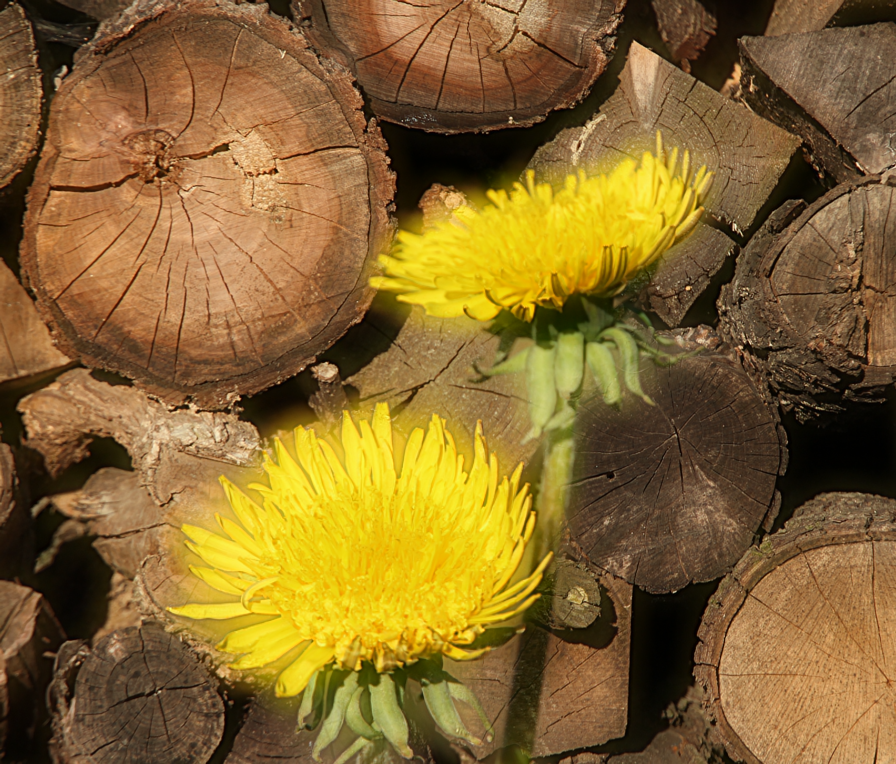}
    \caption{Multifocus  fusion. The input frames focused on the foreground (left) and on the background (middle). The frames were registered by blur-invariant phase correlation and fused by the method from \cite{ZHANG202099} (right).
     } 
    \label{fig:MFF}
\end{figure*}

\subsection{Discussion}

The experiments demonstrate a~very good performance of the proposed invariants in the recognition of blurred objects and in blurred frames registration.
Blur invariants exist in equivalent forms in Fourier domain where they are expressed directly by the projection operator and in the image domain where they use moment expansion. Both domains can be used in experiments and our choice mostly depends on the image size (for large images, Fourier invariants are more efficient and vice versa). 
In terms of recognition power and speed, the proposed invariants are probably the best ``handcrafted'' blur-invariant features ever published.

The comparison to deep-learning methods, represented here by the ResNet CNN, is perhaps even more interesting. We showed that if the scenario is convenient for using ``handcrafted" features, our invariants outperform CNN. By a~convenient scenario, we understand situations, where the number of classes may be high but the classes are relatively small, typically represented by a~single (or very few) training sample(s). To reach a~comparable recognition rate, CNNs require a~massive augmentation over a~wide range of blurs, which makes the training extremely time-consuming.

On the other hand, the proposed invariants can hardly be used for classification into generic classes such as ``person'', ``car'', ``animal'', ``tree'', etc. The invariants do not have the ability to analyze the image content and they are not ``continuous'', which means that two visually similar objects (two dogs or two cars for instance) might have very different invariant values. These scenarios can be well resolved by deep learning, however, there is still the necessity of a~large-scale augmentation of the training set with blur if blurred images are expected on the input of the system.

To summarize, the proposed invariants and CNNs with augmentation are complementary rather than competitive approaches, each of them dominates in distinct situations. 
One of the challenges for future work is to ``fuse'' both approaches for situations that are somewhere in between the above mentioned extremes.

\section{Conclusion}

In this paper, we presented the general theory of invariants with respect to blur. The main original contribution of the paper lies in Theorem~\ref{FTblur_invariants_general}.

The benefit of the paper is twofold. We showed that all previously published examples of blur invariants are just particular cases of a~unified theory, which can be formulated by means of projection operators without a~limitation to a~single blur type. This significantly contributes to the understanding of blur invariants. The application of this theory to the blur types, which have not been fully explored yet, makes it possible to derive new specific blur invariants that would be difficult to construct otherwise.

Several questions, important for the theory and practice of blur invariants, still remain open for future research. A~challenging area is an investigation of linear non-orthogonal projection operators. We have shown that they may generate useful blur invariants in some cases such as directional blur, but we lack a~general theorem similar to GTBI. At the same time, non-orthogonal projectors might provide solutions to many practically important cases where any blur invariants have not be known. Another, even more difficult, open problem is to go beyond linearity and to study blur invariants constructed by means of non-linear projectors. In the case of Gaussian blur, we showed that a~non-linear projector may produce blur invariants in a~natural way. Unlike linear projectors, the non-linear ones have not been consistently investigated, which has been partly due to their variability. 

Another challenge comes from 3D images. Blur invariants in 3D have been explored much less than those in 2D. In 3D, 17 symmetry groups exist~\cite{Weyl} and each of them can create a~blur space. Although the definition of respective projection operators seems to be similar to the 2D case, a~non-trivial problem is to find an appropriate basis $\B$ that separates the moments~\cite{SukFlu:2014}. 

The presented blur invariants, both in Fourier and moment domains, can be made invariant also to rotation, scaling and even to an affine transform. Due to the space limitation, it is not possible to explain these ``combined invariants'' rigorously in this paper.

A way to improving the success rate in recognition of blurred images could be a~fusion of blur invariants with deep learning approaches, which could compensate for weaknesses of both approaches. That could be done either by inserting the invariants into the hidden layers of the network or by decision fusion on the top level. The research on this field is at a~very initial stage and we envisage its dynamic development in the near future.

\section*{Acknowledgement}

This work has been supported by the Czech Science Foundation (GACR) under the project No. GA21-03921S and by the Czech Academy of Sciences under the {\it Praemium Academiae}.

\bibliographystyle{ieeetr}
\bibliography{moments}

\end{document}